\documentclass{article}

\PassOptionsToPackage{numbers, compress}{natbib}

\usepackage[final,nonatbib]{neurips_2019}

\usepackage[utf8]{inputenc} 
\usepackage[T1]{fontenc}    
\usepackage{hyperref}       
\usepackage{url}            
\usepackage{booktabs}       
\usepackage{amsfonts}       
\usepackage{nicefrac}       
\usepackage{microtype}      
\usepackage{algorithm2e}
\usepackage{mathtools}
\usepackage{bbm}
\usepackage[title]{appendix}
\usepackage{cancel}
\usepackage{amsthm}
\usepackage{thmtools,thm-restate}

\DeclarePairedDelimiter\abs{\lvert}{\rvert}%

\RestyleAlgo{algoruled}

\SetKwInput{KwInput}{Input}
\SetKwInput{AKwInput}{Extra Input}
\SetKwInput{KwOutput}{Output}
\SetKwBlock{Repeat}{Repeat k time steps}{}
\SetKwBlock{Rep}{Repeat (k+d-1) time steps}{}
\SetKwInput{KwInitial}{Initialization}
\SetKwInput{AKwInitial}{Extra Initialization}
\newcommand{\A}{\mathbb A}
\newcommand{\E}{\mathop{\mathbb{E}}}

\usepackage{amsmath}
\usepackage{mathtools}
\DeclareMathOperator*{\argmax}{arg\,max}

\title{Stochastic Bandits with Delayed Composite Anonymous Feedback}

\author{Siddhant Garg \thanks{Equal Contributions to this work} \qquad Aditya Kumar Akash \footnotemark[1] \\
\{sgarg33, aakash@wisc.edu\} \\
University of Wisconsin-Madison 
}

\begin{document}

\maketitle

\begin{abstract}
We explore a novel setting of the Multi-Armed Bandit (MAB) problem inspired from real world applications which we call bandits with "\textit{stochastic delayed composite anonymous feedback (SDCAF)}". In SDCAF, the rewards on pulling arms are stochastic with respect to time but spread over a fixed number of time steps in the future after pulling the arm. The complexity of this problem stems from the anonymous feedback to the player and the stochastic generation of the reward. Due to the aggregated nature of the rewards, the player is unable to associate the reward to a particular time step from the past. We present two algorithms for this more complicated setting of SDCAF using phase based extensions of the UCB algorithm. We perform regret analysis to show sub-linear theoretical guarantees on both the algorithms.
\end{abstract}

\vspace{-0.5cm}
\section{Introduction}
\vspace{-0.25cm}
Multi-Armed Bandits (MAB) have been a well studied problem in machine learning theory for capturing the exploration-exploitation trade off in online decision making. MAB has applications to domains like e-commerce, computational advertising,  clinical trials, recommendation systems, etc.

In most of the real world applications, assumptions of the original theoretical MAB model like immediate rewards, non-stochasticity of the rewards, etc do not hold. A more natural setting is when the rewards of pulling bandit arms are delayed in the future since the effects of the actions are not always immediately observed. \cite{pmlr-v80-pike-burke18a} first explored this setting assuming stochastic rewards for pulling an arm which are obtained at some specific time step in the future. This setting is called \textit{delayed, aggregated, anonymous
feedback} (DAAF). The complexity of this problem stems from anonymous feedback to the model due to its inability to distinguish the origin of rewards obtained at a particular time from any of the previous time steps. 

This work was extended by adding a relaxation to the temporal specificity of observing the reward at one specific time in the future by \cite{pmlr-v75-cesa-bianchi18a}. The reward for pulling an arm can now be possibly spread adversarially over multiple time steps in the future. However they made an added assumption on the non-stochasticity of the rewards from each arm, thereby observing the same total reward for pulling the same arm each time. This scenario of non-stochastic composite anonymous feedback (CAF) can be applied to several applications, but it still does not cover the entire spectrum of applications. 

Consider a setting of a clinical trial where the benefits of different medicines on improving patient health are observed. CAF offers a more natural extension to this  scenario than DAAF since the benefits from a medicine can be spread over multiple steps after taking it rather than achieving it all at once at a single time step in the future. However, the improvements effects of the medicine might be different for different patients and thus assuming the same total health improvement for each time using a specific medicine is not very realistic. Inspired from this real world setting, we suggest that a more general bandit setting will be using CAF with the non-stochastic assumption dropped. We study such a MAB setting with stochastic delayed composite anonymous feedback (SDCAF).

For the SDCAF setting, a player has an option to chose one of K actions (bandit arms) at every time step. Once the player picks an action, a reward is generated at random from an underlying reward distribution for that action. Instead of receiving this reward in a single step, it is adversarially spread over this fixed number of time steps after the action was chosen. After every action choice, the player receives the sum total of all the rewards from the previous actions which are due at this particular step. The difficulty of this setting is due to the fact that the player does not know how this aggregated reward has been constituted from the previous actions chosen. Extending algorithms from the theoretical model of SDCAF to practical applications, involves obtaining guarantees on the rewards obtained from them. The regret of the algorithms refers to how much reward was lost on choosing a particular action over the optimal one. We aim to minimize the regret from plays of the bandit.

We present a phase based algorithm for this SDCAF setting which involves running a modified version of the UCB algorithm \cite{Auer:2002:FAM:599614.599677} in phases where the same arm is pulled multiple times in a single phase. This is motivated by the aim to reduce the error in approximating the arm mean due to extra and missing reward components from adjacent arm pulls. We prove sub-linear regret bounds for this algorithm. We also show that a modified version of ODAAF, a phase based improved UCB algorithm proposed in \cite{pmlr-v80-pike-burke18a}, can be used in our setting and achieves sub-linear regret bounds.

\vspace{-0.2cm}
\subsection{Related Work}
\vspace{-0.2cm}
Online learning with delayed feedback has been studied in the non-bandit setting by \cite{Weinberger:2006:DPI:2263240.2266940, Mesterharm2005OnlineLW, Langford:2009:SLF:2984093.2984354,pmlr-v28-joulani13, NIPS2015_5833, Joulani:2016:DOC:3016100.3016143, DBLP:journals/corr/GarrabrantST16} and in the bandit setting by \cite{DBLP:conf/nips/NeuGSA10, pmlr-v28-joulani13, DBLP:conf/aaai/MandelLBP15, DBLP:conf/colt/Cesa-BianchiGMM16, DBLP:conf/uai/VernadeCP17, pmlr-v80-pike-burke18a}. \cite{DBLP:journals/corr/abs-1106-2369} consider contextual stochastic bandits having a reward with a constant delay. The loss function of our setting is a generalization of the loss function of \cite{DBLP:journals/corr/DekelDKP14}. Gaussian process bandits with bounded stochastic delayed rewards were studied in \cite{JMLR:v15:desautels14a}. \cite{pmlr-v80-pike-burke18a} study bandits in the setting of delayed anonymous and aggregated rewards where the rewards are stochastically sampled from a distribution but received at some time step in the future. \cite{pmlr-v75-cesa-bianchi18a} study non-stochastic bandits where the rewards are spread adversarially over some time steps in the future after pulling the arm. 

\vspace{-0.2cm}
\section{Problem Definition}
\vspace{-0.2cm}
We consider a MAB setting with $K>1$ actions or arms in the set $A$ of all actions. At each time step $0\le t \le T$, the player chooses an action $i\in A$ and receives some reward depending on his past and current choices. Each action $i \in A$ is associated with a reward distribution $\nu_i$ which is supported in $[0,1]$, with mean $\mu_i$. Let $R_t(i)$ denote the total reward generated on choosing action $i$ at time step $t$ which is drawn from the distribution $\nu_i$. Note that $R_t(i)$ is not received by the player in its entirety at time step $t$, but rather spread over a maximum of $d$ time steps (including current time $t$) in any arbitrary manner. $R_t(i)$ is defined by the sum $\sum_{s=0}^{d-1}R_{t}^{(s)}(i)$ of $d$ components $R_{t}^{(s)} (i) \ge 0$ for $s = 0, \dots , d -1$, where $R_{t}^{(s)} (i)$ denotes the reward component obtained at time $t+s$ if action $i$ was chosen at time $t$. We refer to choosing an action and pulling an arm interchangeably in our analysis and use a similar notation to \cite{pmlr-v75-cesa-bianchi18a} for uniformity.

We define $X_t$ as the collective reward that the player obtains at time $t$. If the player chose action $i_{t-l}$ at time step $(t-l)$ where $l \in \{0,...,d-1\}$, we can write $X_t = \sum_{l=0}^{d-1} R_{t-l}^l(i_{t-l})$, which is the sum of contributions from all  actions which the player chose in the past. Only actions chosen in the past $(d-1)$ time steps affect the current reward $X_t$ obtained. We have $R_{t-l}^{l}(i) = 0$ for all $i$ and $l$ when $t-l < 0$.

\vspace{-0.3cm}
\section{Algorithms}
\vspace{-0.2cm}
We present two algorithms for this setting of SDCAF in Algorithm \ref{alg1} and Algorithm \ref{alg2} respectively. For Algorithm \ref{alg2} we only specify the additional inputs and initialization over Algorithm \ref{alg1}. We first provide the intuition behind the algorithms and then provide a formal regret analysis. 

Algorithm \ref{alg1} is a modified version of the standard UCB algorithm and is run in phases where the same arm is pulled multiple times along with maintaining an upper confidence bound on the reward from each arm. More specifically, each phase $m$ consists of two steps. In step 1, the arm with maximum upper confidence bound $i = \argmax_j B_j(m-1, \delta)$ is selected. In step 2, the selected arm $i$ is pulled $k$ times repeatedly. We track all time steps where arm $i$ is played till phase $m$ in the set $S_{i}(m)$. The rewards obtained are used to update the running estimate of the arm mean $\hat{\mu}_i(m)$. The intuition behind running the algorithm in phases is to gather sufficient rewards from a single arm so as to have a good estimate of the arm mean reward. This helps us bound the error in our reward estimate due to extra rewards from the previous phase and missing rewards which seep into the next phase due to delay. For every phase of the algorithm, the selected arm is pulled for a fixed number of times $k$. From our regret analysis, setting $k=O(\sqrt{T/\log(T)})$ achieves sub-linear regret.

Algorithm \ref{alg2} is a modification of the ODAAF algorithm proposed in \cite{pmlr-v80-pike-burke18a}, where we remove the \textit{bridge period} as it has no affect on the confidence bounds in the analysis. This is a modified version of the improved-UCB algorithm from \cite{Auer2010} and run in phases where a set $A$ of active arms is maintained, which is pruned based on the confidence on the arm mean estimates. Each phase $m$ consists of three steps. In the first step, an active arm $i \in A_m$ is sampled and then pulled for $n_m-n_{m-1}$ steps. All time steps where arm $i$ was played in the first $m$ phases are collected in the set $S_i(m)$. In the next step, an updated estimate $X_i(m)$ of the arm mean $\mu_i$ is computed. In the final step, the set of active arms is updated by elimination of arm $j$ if the calculated estimate $X_i(m)$ is $\Tilde{\Delta}_m$ smaller than $\max_{j \in A_m} X_j(m)$. The choice of $n_m$ ensures that with good probability the estimates $X_j(m)$ have bounded error. 

We now provide regret analysis for the algorithms and specify the choice of parameters $k$ and $n_m$.

\begin{minipage}{0.50\textwidth}
\begin{algorithm}[H]
\label{alg1}
\caption{Modified UCB}
\SetAlgoLined
\KwInput{A set of arms $\A = \{ 1, \dots, K\}$ \\ \qquad \quad A time horizon $T$ }
\KwInitial{   $S_i(0)=\phi \ \forall i \in \A$ \\ \hspace{2.15cm} $t=0$ \\ \hspace{2.15cm} $m=1 \{\text{Phase}\}$}
 \While{ $(m = 1,2,... ) \cap \ (t \le T)$}{
  \textbf{\underline{Step 1}}\\
  \textbf{If} $|S_i(m-1)|=0$ \textbf{then} \\
  \quad $B_i(m-1, \delta) =  \infty $ \\
  \textbf{Else} \\
  \quad $B_i(m-1, \delta) =  \hat{\mu_{i}}(m-1) + \sqrt{\frac{2log(\frac{1}{\delta}) }{T_i(m-1)}} $\\
  \textbf{Fi} \\
  Choose arm $i \in  \argmax_j B_j(m-1, \delta)$ \\
  \textbf{\underline{Step 2}}\\
  $S_j(m) \leftarrow S_j(m-1), \forall \ j \in \A$ \\
  \textbf{Repeat} k time steps\\
  \qquad Play arm $i$ \\
  \qquad Collect reward $X_t$ at time step $t$ \\
  \qquad $S_i(m) \leftarrow S_i(m) \cup \{t\}$ \\
  \qquad $t \leftarrow t+1$\\
  \textbf{end}\\
  $\hat{\mu_i}(m) = \frac{\sum_{s \in S_i(m)  } X_t}{|S_i(m)|}$ \\
  \vspace{2.5pt}
  $m \leftarrow m+1$
 }
\end{algorithm}
\end{minipage}
\hfill
\begin{minipage}{0.50\textwidth}
\begin{algorithm}[H]
\label{alg2}
\caption{Improved UCB (ODAAF \cite{pmlr-v80-pike-burke18a})}
\SetAlgoLined
\AKwInput{$n_m$ for each phase $m= \{1,2,\dots \}$}
\AKwInitial{ $A_1=A$}
 \While{$(m = 1,2,... ) \cap \ (t \le T)$}{
     \textbf{\underline{Step 1}}: Play arms\\
     \For{$i \in A_m$}{
          $S_j(m) \leftarrow S_j(m-1) , \forall \ j \in \A$\\
          \While{$|S_i(m)|\le n_m$ \textbf{and} $t \le T$}{
                Play arm i \\
                Receive reward $X_t$ at time step $t$\\
                $S_i(m) \leftarrow S_i(m) \cup \{t\}$\\
                $t \leftarrow t+1$\\
          }
      }
      \textbf{\underline{Step 2}}: Eliminate sub-optimal arms\\
      $X_i(m) = \frac{\sum_{t \in S_i(m)}X_t}{|S_i(m)|} , \ \forall \ i \in A_m$ \\ 
      $A_{m+1} \leftarrow A_{m}$\\
      \For{$i \in A_m$}{
        \textbf{If} $ X_i(m) + \Tilde{\Delta}_m < \max_{j \in A_m}X_j(m)$ \\
        \quad $A_{m+1} \leftarrow A_{m+1} - \{i\}$ \\
      }
      \textbf{\underline{Step 3}}: Decrease Tolerance\\
      $\Tilde{\Delta}_{m+1} \leftarrow \frac{\Tilde{\Delta}_m}{2}$ \\
      $m \leftarrow m+1$
 }
\end{algorithm}
\end{minipage}

\subsection{Regret Analysis for Algorithm 1}
The regret analysis closely follows from that of the UCB algorithm described in \cite{Bandit_Algorithms_Book}. Without loss of generality we assume that the first arm is optimal. Thus we have $\mu_1=\mu^*$, and define $\Delta_i = \mu^* - \mu_i$. We assume that the algorithm runs for $n$ phases. Let $T_i(n)=|S_i(n)|$ denote the number of times arm $i$ is played till phase $n$. We bound $\mathbb{E}[T_i(n)]$ for each sub-optimal arm $i$. For this we show that the following good event holds with a high probability bound
$$G_i = \bigg\{ \mu_1 \le \min_{m\in [n]} B_{1}(m,\delta)\bigg\} \cap \bigg\{ \hat{\mu}_{i}(u_i) + \sqrt{\frac{2}{T_i(u_i)}log(\frac{1}{\delta})} \le \mu_{1} \bigg\}$$
Here, $G_i$ is the event that $\mu_1$ is never underestimated by the upper confidence bound of the first arm, while at the same time the upper confidence bound for the mean of arm $i$, after $T(u_i)$ observations are taken from this arm, is below the payoff of the optimal arm. We make a claim that if $G_i$ occurs, then $T_i(n) \le T_i(u_i)$. Since we always have $T_i(n) \le T \; \forall i \in A$, the following holds 
\begin{equation*}
  \E[T_i (n)]=\E[\mathbbm{1}[G_i] T_i (n)] + \E[\mathbbm{1}[G_i^c]T_i(n)] \le T_i(u_i) + P(G_i^c)T  
\end{equation*}
Next we bound the probability of occurrence of the complement event $G_i^c$. We present a lemma to bound the difference between true estimate of mean and the approximate one used in the algorithm.

\begin{restatable}{lemma}{firstLemma}
\label{L1}
If 
$
\bar{\mu}_i(m) = \frac{1}{T_i(m)}\sum_{t\in S_i(m)} R_t(i)     
$ is an unbiased estimator of $\mu_i$ for $m^{th}$ phase, then the error from the estimated mean can be bound as $\abs{\hat{\mu}_{i}(m) - \bar{\mu}_{i}(m)} \le \frac{d}{k}$ where $S_i(m)$ is the set of time steps when arm $i$ was played and $T_i(n)=|S_i(n)|$.
\end{restatable}
 The proof of Lemma \ref{L1} follows from the fact that in each phase the missing rewards from the current phase and extra reward components from the previous phase can be paired up and the maximum difference that we can obtain between them is at most one. We use Lemma-\ref{L1} to bound $P(G_i^c)$ and obtain \textbf{$k=$} $O(\sqrt{T/\log(T)})$. This gives us an upper bound on the number of times a sub-optimal arm is played $\mathbb{E}[T_i(n)] \le \frac{289\log{(T)}}{4\Delta_i^2} + \frac{d}{2}\sqrt{\frac{T}{\log(T)}} + 2$. 

\begin{restatable}{theorem}{firstTheorem}
\label{T1}
For the choice of $k=O(\sqrt{T/\log(T)})$, the regret of Algorithm \ref{alg1} is bounded by $O\bigg(\sqrt{TK\log(T)} + Kd\sqrt{T/\log(T)}\bigg)$.
\end{restatable}
The proof of Theorem \ref{T1} proceeds by plugging in the upper bound on $\mathbb{E}[T_i(n)]$ in the UCB regret analysis. We refer the readers to Appendix \ref{AppendixA} for the detailed regret analysis of Algorithm \ref{alg1} and proofs of Lemma \ref{L1} and Theorem \ref{T1} .

\vspace{-0.2cm}
\subsection{Regret Analysis for Algorithm 2}
\vspace{-0.2cm}
We use the regret analysis from Appendix-F of \cite{pmlr-v80-pike-burke18a} where it is used for the setting of aggregate feedback with bounded delays. A similar analysis works for our setting of SDCAF. For completeness, we restate the analysis components, lemmas and theorems for composite rewards here. We first bound the difference between estimators for the arm mean $\mu_i$
\begin{restatable}{lemma}{secondLemma}
\label{L2}
If $\Tilde{\mu}_i(m)=\frac{1}{T_i(m)}\sum_{t\in S_i(m)}R_t(i)$ is an unbiased estimator for $\mu_i$ in $m^{th}$ phase, then we can bound the difference of $\Tilde{\mu}_i(m)$ with the estimator $X_i(m)$ used in Algorithm \ref{alg2} as $\abs{\Tilde{\mu}_j(m) - X_{i}(m)} \le  \frac{m(d-1)}{n_m} $ where each arm is pulled $n_m$ times till phase $m$, $S_i(m)$ is the set of time steps when arm $i$ was played and $T_i(m)=|S_i(m)|=n_m$,  $X_i(m)$ is the arm mean estimate computed from the delayed reward components.
\end{restatable}
\textbf{Choice of $n_m$}: We use $n_m= O\bigg( \frac{\log(T\Tilde{\Delta}_m^2)}{\Tilde{\Delta}_m^2} + \frac{md}{\Tilde{\Delta}_m} \bigg)$ similar to that in \cite{pmlr-v80-pike-burke18a}. This ensures a small probability for the event that after $m$ phases a suboptimal arm is still in the active set $A_m$. This bounds the regret contribution of all suboptimal arms. The exact expression is given in Appendix \ref{AppendixB}.
\begin{restatable}{theorem}{secondTheorem}
\label{T2}
For the choice of $n_m= O\bigg( \frac{\log(T\Tilde{\Delta}_m^2)}{\Tilde{\Delta}_m^2} + \frac{md}{\Tilde{\Delta}_m} \bigg)$ the regret of Algorithm \ref{alg2} is bounded by $O\bigg(\sqrt{TK\log(K)}+Kd\log(T)\bigg)$.
\end{restatable}
We refer the readers to Appendix \ref{AppendixB} (Appendix-F\cite{pmlr-v80-pike-burke18a}) for the detailed regret analysis and proofs.

\vspace{-0.25cm}
\section{Conclusion and Future Work}
\vspace{-0.2cm}
In this work we explored the setting of stochastic multi-armed bandits with delayed, composite, anonymous feedback. Due to the nature of rewards being stochastic, anonymous feedback to the player, rewards not being available in its entirety and being arbitrarily spread, the problem becomes significantly complex than the standard MAB scenario. We show that simple extensions of the standard UCB and improved UCB algorithms which run in phases can obtain sub-linear regret bounds for this hard setting. We suggest further extensions of our work in two possible directions: the first being analysing the case when delay parameter $d$ is not perfectly known, and the second being considering a similar setting for the contextual bandits.

\bibliography{neurips_2019}
\bibliographystyle{plain}
\begin{appendices}

\section{Regret Analysis for Algorithm 1}
\label{AppendixA}
Let $\mu_i \  \forall i \in A$, represent the means of the reward distributions $\nu_i$. Without loss of generality we assume that the first arm is optimal so that $\mu_1=\mu^*$. We define $\Delta_i = \mu^* - \mu_i$. Algorithm 1 runs in phases of pulling the same arm for $k$ time steps and thus the regret over $n$ phases can be written as 
\begin{equation}
\label{EE1}
    \mathcal{R}_n = \sum_{i=1}^k \Delta_i\mathbb{E}[T_i(n)]
\end{equation}
Where $T_i(n)$ denotes number of times arm $i$ was played in $n$ phases. We bound the $\mathbb{E}[T_i(n)]$ for each sub-optimal arm $i$. 
Let $G_i$ to be a good event for each arm $i$ defined as follows
$$G_i = \bigg\{ \mu_1 \le \min_{m\in [n]} B_{1}(m,\delta)\bigg\} \cap \bigg\{ \hat{\mu}_{i}(u_i) + \sqrt{\frac{2}{T_i(u_i)}log(\frac{1}{\delta})} \le \mu_{1} \bigg\}$$
where $u_i \in [n]$ is a constant to be chosen later. So $G_i$ is the event that $\mu_1$ is never underestimated by the upper confidence bound of the first arm, while at the same time the upper confidence bound for the mean of the arm $i$ after $T(u_i)$ observations are taken from this arm is below the payoff of the optimal arm. Two things are shown :
\begin{itemize}
    \item If $G_i$ occurs, then $T_i(n) \le T_i(u_i)$.
    \item Low probability of occurrence for the complement event $G^c$.
\end{itemize}
Since we always have $T_i(n) \le T \; \forall i \in A$, following holds 
\begin{equation}
\label{ETn}
  \E[T_i (n)]=\E[\mathbbm{1}[G_i] T_i (n)] + \E[\mathbbm{1}[G_i^c]T_i(n)] \le T_i(u_i) + P(G_i^c)T  
\end{equation}

In the following step we assume that $G_i$ holds. Now we show that $T_i(n) \le T_i(u_i)$. Assume $T_i(n) > T_i(u_i)$. Then arm $i$ was played more that $T_i(u_i)$ times over $n$ phases, so there must exist a phase $t \in [n]$ where $T_i(t-1) = T_i(u_i)$ and $A_t=i$. But using how $G_i$ is defined, we have $B_i(t-1, \delta) < B_1(t-1, \delta)$. Hence $A_t = \argmax_j B_j(t-1, \delta) \ne i$, a contradiction. So if $G_i$ occurs, $T_i(n) \le T_i(u_i)$.

Now we bound $\mathbb{P}(G_i^c)$. The event $G_i^c$ is as follows:
\begin{equation}
\label{Gc}
    G_i^c = \underbrace{\bigg\{ \mu_1 \ge \min_{m\in [n]} B_{1}(m,\delta)\bigg\}}_{I} \cup \underbrace{\bigg\{ \hat{\mu}_{i}(u_i) + \sqrt{\frac{2}{T_i(u_i)}log(\frac{1}{\delta})} \ge \mu_{1} \bigg\}}_{II}
\end{equation}

Using a union bound the probability of term $I$ of $G_i^c$ can be upper bounded as 
$$P(I) = P\bigg( \mu_1 \ge \min_{m\in [n]} B_{1}(m,\delta)\bigg) \le \sum_{m=1}^{n} P\bigg( \mu_1 \ge \hat{\mu}_{1}(m) + \sqrt{\frac{2log(\frac{1}{\delta})}{T_1(m)}} \bigg)$$
\firstLemma*
\begin{proof}
Consider the following estimator for the mean of the rewards generated from $i^{th}$ arm till $m$ phases :
$$
\bar{\mu}_i(m) = \frac{\sum_{t\in S_i(m)} R_t(i)}{T_i(m)}     
$$ 
where $T_i(m) = |P_i(m)|$. It can be seen $\mathbb{E}[\bar{\mu}_i(m)] = \mu_i$. \\
If arm $i$ was played in phase $m$, then we have 
$$
\bigg|\sum_{t \in S_i(m)\setminus S_i(m-1)} (R_t(i) - X_t)\bigg| \le d
$$
where $d$ is the delay parameter over which the rewards are distributed. Because the missing and extra reward components can be paired up and the maximum difference we can obtain is at most one.

After $m$ phases, suppose an arm $i$ was played $z$ times. Then we can bound  
\begin{equation}
\label{diff}
    \abs{\hat{\mu}_{i}(m) - \bar{\mu}_{i}(m)} \le \frac{d \times \cancel{z}}{k \times \cancel{z}} \le \frac{d}{k}
\end{equation} 
since in each phase an arm is pulled $k$ times. This gives $\hat{\mu}_{1}(m) \le \bar{\mu}_{1}(m) + \frac{d}{k}$ and $\hat{\mu}_{1}(m) \ge \bar{\mu}_{1}(m) - \frac{d}{k}$. 
\end{proof}
Plugging this in our bound for $P(I)$ gives
$$
P(I) \le \sum_{m=1}^{n} P\bigg( \mu_1 \ge \bar{\mu}_{1}(m) - \frac{d}{k} + \sqrt{\frac{2log(\frac{1}{\delta})}{T_1(m)}} \bigg) 
$$

We then choose $k$ such that following holds for all $m$
$$
- \frac{d}{k} + \sqrt{\frac{2log(\frac{1}{\delta})}{T_1(m)}} \ge a\sqrt{\frac{2log(\frac{1}{\delta})}{T_1(m)}}
$$
Since $T_1(m) \le T \; \forall m$, $k$ is selected as 
\begin{equation}
\label{kchoice}
    k = \frac{d}{(1-a)}\sqrt{\frac{T}{2log(\frac{1}{\delta})}}
\end{equation}
Using this choice of $k$ and the fact that rewards are obtained from distributions which are subgaussian, we bound $P(I)$ further as follows
\begin{equation}
\label{PI}
\begin{split}
    P(I) &\le \sum_{m=1}^{n} P\bigg( \mu_1 \ge \bar{\mu}_{1}(m) + a\sqrt{\frac{2log(\frac{1}{\delta})}{T_1(m)}} \bigg)\\
    &\le \sum_{m=1}^n \delta^{a^2} = n\delta^{a^2}
\end{split}
\end{equation}

The next step is to bound the probability of term $II$ in (\ref{Gc}). Note $\mu_1 = \mu_i + \Delta_i$. Using (\ref{diff}) we get 
\begin{equation*}
    \begin{split}
        P(II) &= P\bigg( \hat{\mu}_{i}(u_i) + \sqrt{\frac{2log(\frac{1}{\delta})}{T_i(u_i)}} \ge \mu_i \bigg)\\
        &= P\bigg( \hat{\mu}_{i}(u_i) + \sqrt{\frac{2log(\frac{1}{\delta})}{T_i(u_i)}} \ge \mu_i + \Delta_i  \bigg)\\
        &\le P\bigg( \bar{\mu}_{i}(u_i) + \frac{d}{k} +  \sqrt{\frac{2log(\frac{1}{\delta})}{T_i(u_i)}} \ge \mu_i + \Delta_i  \bigg)\\
        P(II) &\le P\bigg( \bar{\mu}_{i}(u_i) - \mu_i \ge \Delta_i - \frac{d}{k} -  \sqrt{\frac{2log(\frac{1}{\delta})}{T_i(u_i)}}  \bigg)
    \end{split}
\end{equation*}
Because of our choice of $k$ in (\ref{kchoice}) we have 
\begin{equation*}
    \begin{split}
        \frac{d}{k} &\le (1-a) \sqrt{\frac{2log(\frac{1}{\delta})}{T}}\\
        &\le (1-a) \sqrt{\frac{2log(\frac{1}{\delta})}{T_i(u_i)}}\\
        \frac{d}{k} &\le \sqrt{\frac{2log(\frac{1}{\delta})}{T_i(u_i)}}
    \end{split}
\end{equation*}
Now we show that $u_i$ can be chosen in some sense such that following inequality holds
\begin{equation}
\label{uchoice}
    \begin{split}
        \Delta_i - \frac{d}{k} -  \sqrt{\frac{2log(\frac{1}{\delta})}{T_i(u_i)}} &> c\Delta_i\\
        (1-c)\Delta_i &> \frac{d}{k} +  \sqrt{\frac{2log(\frac{1}{\delta})}{T_i(u_i)}}\\
        (1-c)\Delta_i &> 2\sqrt{\frac{2log(\frac{1}{\delta})}{T_i(u_i)}}\\
        T_i(u_i) &> \frac{8log(\frac{1}{\delta})}{(1-c)^2\Delta_i^2}
    \end{split}
\end{equation}
We assume that arm $i$ is played in $g_i$ number of phases. Hence $T_i(u_i) = g_ik$. This gives us that we can choose $g_i = \bigg\lceil \frac{ 8\sqrt{2} (1-a) \log(\frac{1}{\delta})^{1.5}}{ (1-c)^2\Delta_i^2 d \sqrt{T} } \bigg\rceil$.
Using this choice of $g_i$ and the sub-gaussian assumption we can bound $P(II)$ 
\begin{equation}
    \label{PII}
    P(II) \le P(\mu_i(u_i) - \mu_i > c\Delta_i) \le \exp{\bigg(\frac{-c^2\Delta_i^2T_i(u_i)^2}{2}\bigg)}
\end{equation}
Taking (\ref{PI}) and (\ref{PII}), we have
\begin{equation*}
    P(G_i^c) \le n\delta^{a^2} + \exp{\bigg(\frac{-c^2\Delta_i^2T_i(u_i)^2}{2}\bigg)}
\end{equation*}
When substituted in (\ref{ETn}) we obtain 
\begin{equation}
\label{ETn2}
    \mathbb{E}[T_i(n)] \le T_i(u_i) + T\bigg(n\delta^{a^2} + \exp{\bigg(\frac{-c^2\Delta_i^2T_i(u_i)^2}{2}\bigg)}\bigg)
\end{equation}
Making the assumption that $\delta^{a^2} = \frac{1}{T^2}$ and the choice of $g_i$ from (\ref{uchoice}), equation (\ref{ETn2}) leads to 
\begin{equation}
\label{ETn3}
\begin{split}
    \mathbb{E}[T_i(n)] &\le \frac{16\log{(T)}}{a^2(1-c)^2\Delta_i^2} + k + 1 + T^{1 - \frac{16c^2}{(1-c)^2a^2}}\\
    &= \frac{16\log{(T)}}{a^2(1-c)^2\Delta_i^2} + \frac{d}{(1-a)}\sqrt{\frac{T}{2log(\frac{1}{\delta})}} + 1 + T^{1 - \frac{16c^2}{(1-c)^2a^2}}
\end{split}
\end{equation}
Now we make the choice of $a, c \in (0, 1)$. We choose $a = \frac{1}{2}$, and accordingly choose $c$ such that last term in (\ref{ETn3}) does not contribute in polynomial dependence. We choose $c=\frac{1}{17}$, so that $\frac{16c^2}{(1-c)^2a^2}=\frac{1}{4}$. This leads to 
\begin{equation}
\label{ETinFinal}
    \mathbb{E}[T_i(n)] \le \frac{289\log{(T)}}{4\Delta_i^2} + \frac{d}{2}\sqrt{\frac{T}{\log(T)}} + 2
\end{equation}

\firstTheorem*
\begin{proof}
From (\ref{ETinFinal}) we have that for each sub-optimal arm $i$ we can bound 
$$
\mathbb{E}[T_i(n)] \le \frac{289\log{(T)}}{4\Delta_i^2} + \frac{d}{2}\sqrt{\frac{T}{\log(T)}} + 2
$$
Now using regret definition from (\ref{EE1}) we have 
\begin{equation*}
    \begin{split}
        \mathcal{R}_n &= \sum_{i=1}^K \Delta_i \mathbb{E}[T_i(n)] = \sum_{i:\Delta_i < \Delta} \Delta_i \mathbb{E}[T_i(n)] + \sum_{i:\Delta_i \ge \Delta} \Delta_i \mathbb{E}[T_i(n)]\\
        &\le T\Delta + \sum_{i:\Delta_i \ge \Delta} \bigg( \frac{289\log{(T)}}{4\Delta_i} + \frac{d\Delta_i}{2}\sqrt{\frac{T}{\log(T)}} + 2\Delta_i \bigg)\\ &\le T\Delta + \frac{289K\log{(T)}}{4\Delta} + \bigg( \frac{d}{2}\sqrt{\frac{T}{\log(T)}} + 2 \bigg)\sum_i\Delta_i\\
        &\le 17\sqrt{TK\log{(T)}} + \frac{d\sum_i\Delta_i}{2}\sqrt{\frac{T}{\log(T)}} + 2\sum_i\Delta_i
    \end{split}
\end{equation*}
The first inequality comes from $\sum_{i:\Delta_i < \Delta}T_i(n) \le T$ and the last from choice of $\Delta = \sqrt{\frac{289K\log{(T)}}{4T}}$. The term $\sum_i\Delta_i$ can be upper bounded by $K$ since each $\Delta_i \le 1$. Thus we get the regret bound $O\bigg(\sqrt{TK\log(T)} + Kd\sqrt{T/\log(T)}\bigg)$.
\end{proof}


\section{Regret Analysis for Algorithm 2}
\label{AppendixB}
The regret analysis analysis for this algorithm is taken verbatim from Appendix-F of \cite{pmlr-v80-pike-burke18a} with minor modifications. The same analysis can be used to obtain sub-linear regret bounds for the problem setting with composite delayed rewards. Since our SDCAF problem setting and algorithm \ref{alg2} is slightly different from ODAAF, we include this section here for completeness.

\secondLemma*
\begin{proof}
Since the rewards are spread over $d$ time steps in an adversarial way, in the worst case the first $d-1$ rewards collected for arm $j$ in phase $m$ would have components from previous arms. Similarly for the last $d-1$ arm pulls, the reward components would seep into the next arm pull. Defining $F_{i,j}$ and $L_{i,j}$ as the first and last points of playing arm $j$ in phase $i$, we have 
\begin{equation}
    \label{eq_diff}
    \left| \sum_{t=F_{i,j}}^{L_{i,j}}R_t(j) - \sum_{t=F_{i,j}}^{L_{i,j}}X_t\right| \le (d-1)
\end{equation}

because we can pair up some of the missing and extra reward components, and in each pair the difference is at most one. Then since $S_j(m) = \cup_{i=1}^{m}\{F_{i,j}, F_{i,j}+1, \dots,L_{i,j}\}$ and using (\ref{eq_diff}) we get 

\begin{equation}
    \frac{1}{n_m}\left| \sum_{t\in S_j(m)}R_t(j) - \sum_{t\in S_j(m)}X_t\right| \le \frac{m(d-1)}{n_m}.
\end{equation}

Define $\Tilde{\mu}_j(m)=\frac{1}{T_j(m)}\sum_{t\in S_j(m)}R_t(j)$ and recall that $X_j(m) = \frac{1}{T_j(m)}\sum_{t\in S_j(m)}X_t$, where $T_j(m) = |S_j(m)|$. 
\end{proof}

\begin{restatable}{lemma}{thirdLemma}
\label{L3}
For the above choice of $n_m$, with high probability $\ge \bigg(1-\frac{2}{T\Tilde{\Delta}_m^2}\bigg)$, either arm $j$ is eliminated after phase $m$ or it is still active i.e $X_i(m) - \mu_i \le \Tilde{\Delta}_m/2$.  
\end{restatable}
\begin{proof}
For any $a > \frac{m(d-1)}{n_m}$,
\begin{equation*}
    \begin{split}
        P(|X_j(m)-\mu_j|>a) &\le P(|X_j(m)-\Tilde{\mu}_j(m)|+|\Tilde{\mu}_j(m)-\mu_j|>a) \\
        &\le P\bigg(|\Tilde{\mu}_j(m)-\mu_j|>a-\frac{m(d-1)}{n_m}\bigg) \\
        &\le 2\exp\left\{-2n_m\left(a-\frac{m(d-1)}{n_m}\right)^2\right\}
    \end{split}
\end{equation*}
where the first inequality is from triangle inequality and the last from Hoeffding's inequality since $R_t(j) \in [0,1]$ are independent samples from $\nu_j$, the reward distribution of arm $j$. In particular choosing $a = \sqrt{\frac{\log(T\Tilde{\Delta}_m^2)}{2n_m}} + \frac{m(d-1)}{n_m}$ guarantees that $P(|X_j(m)-\mu_j|>a)\le \frac{2}{T\Tilde{\Delta}_m^2}$.\\
Setting 
$$
n_m = \left\lceil \frac{1}{2\Tilde{\Delta}_m^2}\left(\sqrt{\log(T\Tilde{\Delta}_m^2)}+\sqrt{\log(T\Tilde{\Delta}_m^2)+4\Tilde{\Delta}_mm(d-1)}\right)^2 \right\rceil
$$ ensures that $P(|X_j(m)-\mu_j|<\frac{\Tilde{\Delta}_m}{2}) \ge 1-\frac{2}{T\Tilde{\Delta}_m^2}$.
\end{proof}

\secondTheorem*
\begin{proof}
Using Theorem 32 from \cite{pmlr-v80-pike-burke18a}, which uses analysis of improved UCB from \cite{Auer2010} we substitute the value of $n_m$ to get following bound on regret
\begin{center}
$\sum_{\substack{j \in A \\ \Delta_j > \lambda}}\left(\Delta_j + \frac{64\log(T\Delta_j^2)}{\Delta_j} + 64\log(\frac{2}{\Delta_j})(d-1) + \frac{96}{\Delta_j}\right) + \sum_{\substack{j \in A \\ \Delta_j < \lambda}}\frac{64}{\lambda} + T\max_{\substack{j \in A \\ \Delta_j \le \lambda}}\Delta_j$
\end{center}
In particular, optimizing with respect to $\lambda$ gives the worst case regret of $O(\sqrt{KT\log(K)}+Kd\log(T))$ which is sublinear in $T$.
\end{proof}

\end{appendices}

\end{document}